\documentclass[11pt]{article}
\usepackage[final]{acl}
\usepackage{xcolor}
\usepackage{times}
\usepackage{latexsym}
\usepackage[most]{tcolorbox}
\usepackage[T1]{fontenc}
\usepackage{times}
\usepackage{latexsym}
\usepackage{graphicx}
\usepackage{hyperref}
\usepackage{fancyvrb} 
\usepackage{tcolorbox} 
\usepackage{booktabs}
\usepackage{enumitem}
\usepackage{algorithm,algpseudocode}

\usepackage{amsmath,amssymb,mathtools,amsthm}

\newtheorem{proposition}{Proposition}

\newcommand{\bx}{{\mathbf{x}}}
\newcommand{\by}{{\mathbf{y}}}

\usepackage[utf8]{inputenc}

\usepackage{microtype}

\usepackage{inconsolata}

\usepackage{graphicx}

\newcommand{\SOVER}{\texttt{SOVER}}
\newcommand{\zt}{\textsf{Z3}}
\newcommand{\dReal}{\textsf{dReal}}
\newcommand{\EM}{\textsf{EquivaMap}}

\usepackage{amsmath}
\DeclareMathOperator*{\argmin}{arg\,min}

\title{SOVER: Formal Certification of Optimization Reformulations via LLM-Assisted SMT Verification}

\author{Swapnil Bhattacharyya\\
  TCS Research, Mumbai\\
  \texttt{b.swapnil@tcs.com} \\\And
  Mayank Baranwal \\
  TCS Research, Mumbai\\
  Indian Institute of Technology Bombay\\
  \texttt{mbaranwal@iitb.ac.in} \\}

\begin{document}
\maketitle
\begin{abstract}
Large Language Models (LLMs) have shown remarkable promise in translating and reformulating complex mathematical optimization problems across modeling languages. However, validating such transformations through empirical solver executions alone is unreliable, as solver outcomes may be affected by local minima, structural timeouts, numerical artifacts, and subtle semantic divergence between formulations. We introduce \SOVER, an LLM-assisted SMT framework that separates semantic mapping from formal certification: \zt{} checks domain cross-feasibility and global objective-order preservation for mixed-integer linear formulations, while \dReal{} provides tolerance-aware feasibility/range and $\epsilon$-argmin checks for continuous nonlinear formulations. We also introduce \textsc{NLEquiv-150}, a public benchmark of 100 equivalent and 50 deliberately hard non-equivalent nonlinear reformulation pairs. With LLM-extracted mappings, \SOVER{} classifies 149/150 pairs (99.33\%) correctly, including all 50 hard negatives; the sole error is an incomplete mapping extraction.
\end{abstract}

\section{Introduction}\label{sec:Intro}
Optimization is central to scientific and engineering decision-making, with applications in energy-cost reduction, supply-chain planning, and profit maximization \cite{singh2012overview}. Combinatorial optimization \cite{le2024survey} further underpins core problems in operations research and computer science, from network routing \cite{DBLP:journals/mp/ChristofidesMT81} to machine learning system design \cite{sun2019survey}. Understanding when two optimization problems are equivalent, or when one reduces to another \cite{karp2009reducibility}, is therefore fundamental: it enables structurally related NP-complete problems \cite{aaronson2005guest}, such as SAT and the Traveling Salesperson problem \cite{pop2024comprehensive}, to be organized into equivalence classes \cite{paulson2006defining} and solved using transferable algorithmic ideas \cite{liu2022teachingnetworkssolveoptimization}.

LLMs are increasingly used as optimization copilots, translating natural language specifications into formal models \cite{huang2025llms} and assisting in solution workflows \cite{zhang2025systematic}. However, LLM-generated formulations remain unreliable \cite{ma2026large}, often containing ambiguous terminology, missing assumptions, inconsistent variables, or structural errors \cite{ahmaditeshnizi2024optimus}. Rigorous equivalence checking is therefore essential \cite{wang2025large}, especially when AI-generated models \cite{yao2025fact} must be validated against trusted ground-truth formulations \cite{zhou2025step} while preserving intended mathematical semantics \cite{xiao2025survey}.

\begin{figure*}[!ht]
	\begin{center}
		\begin{tabular}{cc}
			\includegraphics[width=0.94\columnwidth]{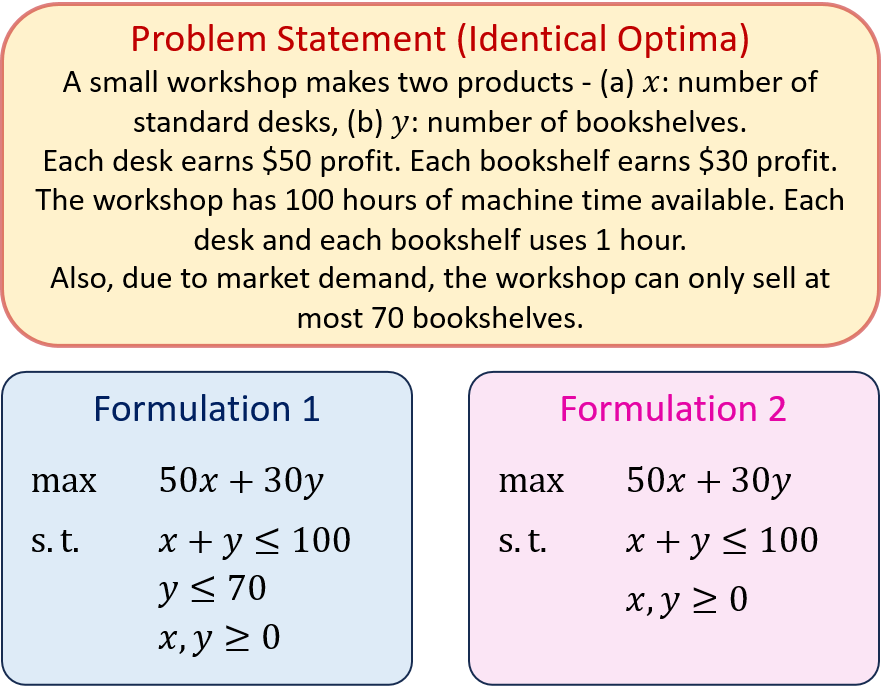} & \includegraphics[width=0.94\columnwidth]{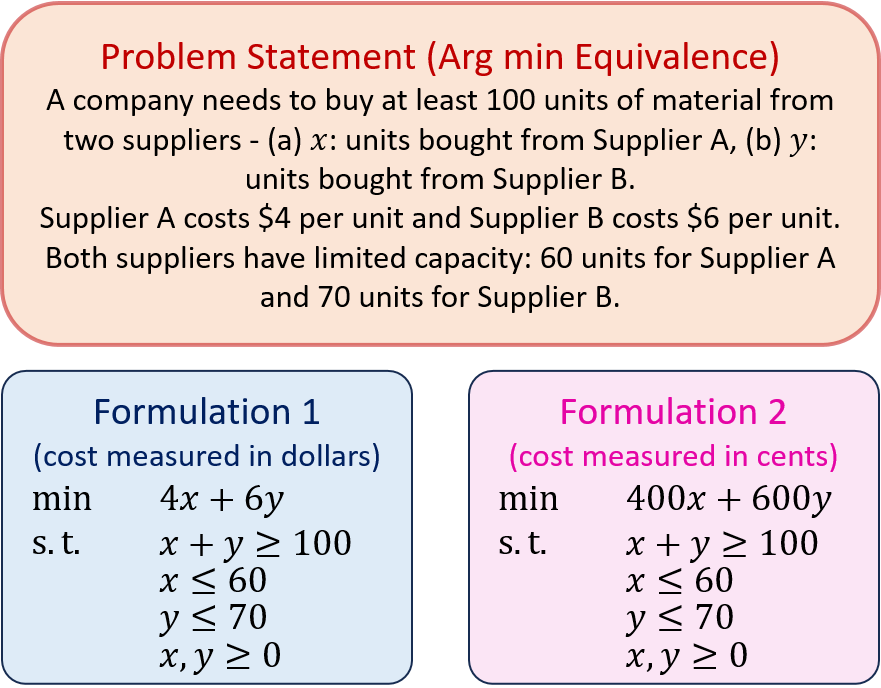} \cr
			(a) Identical optimum scenario & (b) Diff. optimal values, yet $\argmin$ equivalent \cr
			\includegraphics[width=0.94\columnwidth]{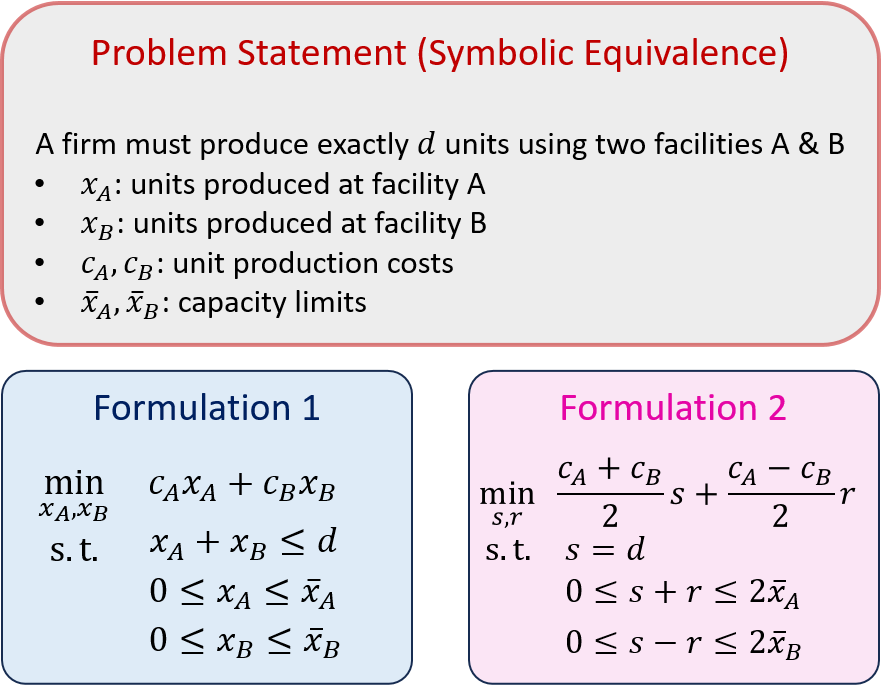} & \includegraphics[width=0.94\columnwidth]{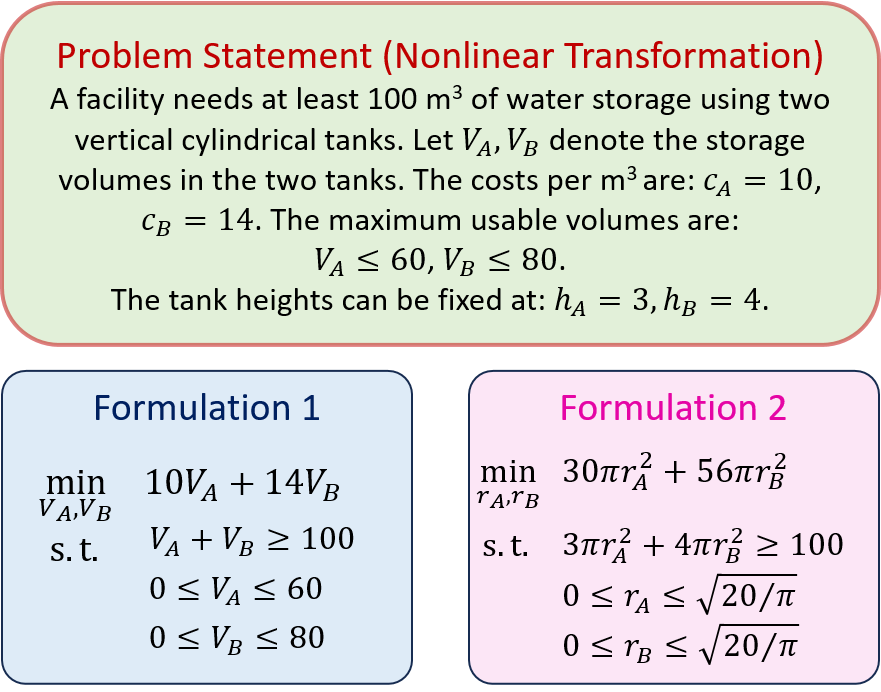} \cr
			(c) Symbolic formulations & (d) Nonlinear transformation
		\end{tabular}
	\end{center}
	\caption{\textbf{Challenges in assessing optimization reformulations}. (a) Existing methods may fail to distinguish genuinely different problems when inactive constraints yield the same optimizer and objective value. (b) Conversely, they may incorrectly flag equivalent formulations with scaled objectives as different. (c) Robust reformulation analysis should also be parameter-agnostic, and (d) capable of recognizing equivalence across nonlinear representations.}
	\label{fig:challenges}
    \vspace{-1em}
\end{figure*}

A central challenge is that reformulation equivalence is not captured by equality of solver outputs. As shown in Figure~\ref{fig:challenges}, different formulations may share the same optimizer and optimal value when constraints are inactive, while equivalent formulations can differ when objectives are scaled. Robust verification must also handle symbolic parameters and recognize equivalence across heterogeneous representations, such as an LP becoming an NLP under a variable transformation. Thus, optimal-value comparison, sampled solutions, and syntactic similarity produce both false positives/negatives.

% A central challenge is that reformulation equivalence is not captured by equality of solver outputs. As illustrated in Figure~\ref{fig:challenges}, genuinely different formulations may share the same optimizer and optimal value when some constraints are inactive, while equivalent formulations may appear different when their objectives differ by a positive scaling. Robust verification should also handle symbolic parameters and recognize equivalence across heterogeneous representations, including cases where an LP in one physical variable becomes an NLP under another. Thus, optimal-value comparison, sampled solutions, and syntactic similarity can produce both false positives and false negatives.

Recent work has begun to study LLM-assisted equivalence checking. In particular, \EM~\cite{zhai2025equivamap} uses LLMs to infer transformations between decision variables and then verifies that the mapped formulations preserve feasibility and optimality. While this is an important step toward automated reformulation analysis, the broader challenges in Figure~\ref{fig:challenges} require verification that is robust to inactive constraints, scaled or monotone objectives, symbolic parameters, and heterogeneous linear/nonlinear representations. Traditional testing cannot exhaustively cover these semantic boundaries \cite{mahmoud2024formal}, and the brittleness of generative modeling further motivates provable validation mechanisms \cite{chen2026solver} that separate structural interpretation from logical verification \cite{pei-etal-2025-fover}.

We propose \SOVER, a hybrid framework for provably correct optimization reformulation. \SOVER{} uses LLMs to parse optimization statements, align variables and parameters, and propose symbolic correspondences. It then discharges verification obligations using SMT solvers under a shared logical representation. Instead of relying on agreement over finitely many solver executions, \SOVER{} checks domain cross-feasibility and global objective-order preservation, thereby establishing $\arg\min$ equivalence. Moreover, \SOVER{} solves only the trusted source optimization problem and verifies reformulation correctness through SMT-based feasibility and ordering checks, enabling efficient validation across mixed-integer linear, continuous nonlinear, and non-convex settings.

\noindent This paper makes the following contributions:\vspace{-.5em}
\begin{itemize}[leftmargin=*]
    \item \textbf{Formal reformulation verification.} We introduce \SOVER, an LLM-assisted SMT framework that encodes reformulation correctness as logical verification conditions over symbolic domains.\vspace{-1em}
    \item \textbf{Equivalence beyond objective-value matching.} \SOVER{} verifies domain cross-feasibility \& global objective-order preservation, avoiding failures from inactive constraints or scaled objectives.\vspace{-1em}
    \item \textbf{Parameter-agnostic and heterogeneous checking.} \SOVER{} supports symbolic parameters and verifies equivalence across different mathematical representations, including LP-to-NLP reformulations induced by variable transformations.\vspace{-1em}
    \item \textbf{Solver-efficient verification.} \SOVER{} requires solving only the trusted source problem and uses SMT queries to verify feasibility and ordering obligations, rather than validating reformulations through independent optimal-value comparisons. On EquivaFormulation, \zt{} verification averages 0.03\,s per pair; end-to-end latency is dominated by LLM-assisted mapping.\vspace{-1em}
    \item \textbf{Support for mixed-integer and nonlinear settings.} The framework integrates \zt{} for mixed-integer linear formulations and \dReal{} for continuous nonlinear formulations via $\delta$-satisfiability.
    \item \textbf{\textsc{NLEquiv-150} benchmark.} We release 100 equivalent and 50 hard non-equivalent application-grounded nonlinear pairs spanning objective, constraint, equality/direction/coefficient, and transformation-range failures.
\end{itemize}

\section{Related Work}\label{sec:RelatedWork}
\textbf{Large Language Models and Tool Integration.}
Recent advances in natural language processing have produced Large Language Models (LLMs) that perform strongly across question answering \cite{brown2020language}, summarization, translation \cite{grattafiori2024llama,team2023gemini}, and code generation \cite{chen2021evaluating}. Their capabilities are further strengthened through integration with external tools for code execution, debugging, and iterative refinement \cite{qin2024toolllm}.

\noindent\textbf{LLMs in Mathematical Optimization.}
Recent work has explored LLMs in operations research and mathematical optimization, either by generating solutions directly through iterative prompting \cite{yang2024large} or by assisting modeling workflows. \cite{chen2024diagnosing} developed a chatbot for diagnosing and repairing infeasible Pyomo models, while OptiMUS \cite{ahmaditeshnizi2024optimus} translates natural language descriptions into MILP formulations and debugs solver code through automated testing. These efforts reflect a larger push toward foundation models and self-improving agents for diverse optimization instances \cite{li2025towards,zhou2025steporlm}.

\noindent\textbf{Formal Equivalence Checking and Verification.}
As LLM-based optimization tools evolve from assistants to automated formulation engines \cite{li2025optimization}, verification becomes a central bottleneck. Generative models may produce syntactically different formulations for the same problem \cite{herrera2025overview}, which makes surface-level matching insufficient. Building on Karp reductions, \EM~\cite{zhai2025equivamap} uses LLMs to infer mappings between decision-variable spaces and applies lightweight verification to assess structural equivalence. More broadly, the safe deployment of generative AI in decision-making pipelines \cite{shi2024large} requires the coupling of LLM-based modeling with rigorous verification. Our work follows this direction by combining LLM-assisted semantic alignment with SMT-based formal checks for LPs, MILPs, and nonlinear reformulations.

\section{Preliminaries}\label{sec:prelim}

\subsection{Constrained Mathematical Optimization}\label{subsec:ConsOpt}
A constrained minimization problem is written as
\[
    \min_{\bx\in\mathcal{X}} f(\bx),
\]
where $f:\mathbb{R}^n\to\mathbb{R}$ is the objective and $\mathcal{X}\subseteq\mathbb{R}^n$ is the feasible region induced by the problem constraints. When comparing formulations, equality of optimal objective values is not sufficient: the relevant object is often the optimal solution set,
\[
    \argmin_{\bx\in\mathcal{X}} f(\bx)
    \!=\!
    \left\{
    \bx^\star\!\in\!\mathcal{X}
    \;\middle|\;
    f(\bx^\star)\!\leq\!f(\bx),\ \forall \bx\!\in\!\mathcal{X}
    \right\}\!.
\]
Two formulations are \emph{$\argmin$ equivalent} if a mapping $\Sigma$ between their feasible domains maps the minimizers of one formulation onto the minimizers of the other. This allows equivalent formulations to have different numerical optimum values, for example when one objective is a positive scaling or strictly increasing transformation of the other.
% Two formulations are said to be \emph{$\argmin$ equivalent} if there exists a mapping $\Sigma$ between their feasible domains that maps the optimal solution set of one formulation onto the optimal solution set of the other. This notion is stronger and more structurally meaningful than comparing only objective values. In particular, two formulations may be $\argmin$ equivalent even when their minimum objective values differ. For example, multiplying the objective by a positive constant changes the numerical value of the optimum but leaves the set of minimizers unchanged. More generally, any strictly increasing transformation of the objective preserves the ordering of feasible points and therefore preserves the $\argmin$ set, even though it may alter the scalar optimum.

\subsection{Satisfiability Modulo Theories (SMT)}\label{subsec:SMT}
SMT extends Boolean satisfiability to first-order formulas over background theories such as real arithmetic, integer arithmetic, arrays, bit-vectors, and uninterpreted functions. In verification, SMT solvers prove a property $\varphi$ by checking whether its negation $\neg\varphi$ is satisfiable. A \texttt{SAT} result gives a concrete counterexample, while \texttt{UNSAT} certifies that no counterexample exists within the specified theory. This makes SMT well suited to reformulation verification: instead of relying on sampled instances or heuristic solver behavior, the verifier directly asks whether any feasible assignment violates the claimed equivalence.

\section{Proposed Methodology}\label{sec:Method}

\SOVER{} verifies whether two optimization formulations represent the same underlying optimization problem under a candidate reformulation map. The framework separates the task into two components: an LLM-assisted mapping module proposes correspondences between variables and parameters, while an SMT-based verifier checks whether these correspondences preserve feasible regions and objective order. This design allows the LLM to assist with structural interpretation, while all equivalence claims are discharged by formal logical queries.

\subsection{Mapping Generation}\label{subsec:Mapping}

Verifying reformulation equivalence requires aligning the semantic components of the two models. Given a source formulation $P_A$ and a target formulation $P_B$, \SOVER{} constructs candidate mappings for both decision variables and problem parameters. These mappings define the substitutions used by the SMT verifier to express both formulations in a shared symbolic coordinate system.

\begin{table*}[htbp]
\centering
\begin{tabular}{llc}
\hline
\textbf{Variation ID} & \textbf{Variation Type} & \textbf{Equivalent} \\ 
\hline
\_c & Rename parameters and variables & Yes \\
\_d & Binary substitution of decision variables & Yes \\
\_f & Replace objective term by an equivalent constraint & Yes \\
\_g & Add slack variables & Yes \\
\_h & Linear substitution of decision variables & Yes \\
\_i & Rescale the objective function & Yes \\
\_j & Replace the formulation by an unrelated formulation & No \\
\_k & Convert an unrelated formulation into a feasibility problem & No \\
\_l & Drop constraints that are inactive at the observed optimum & No \\
\_m$^{+}$ & Nonlinear/non-convex coordinate reformulations & Yes\\
\_m$^{-}$ & Near-equivalent nonlinear reformulations with subtle semantic mismatch & No\\
\hline
\end{tabular}
\caption{Taxonomy of problem variations and expected reformulation equivalence. Rows \_c--\_l follow \cite{zhai2025equivamap}; the $\_m^{+}$ and $\_m^{-}$ nonlinear categories are introduced in \textsc{NLEquiv-150}.}
\label{tab:variations}
\end{table*}

\subsubsection{Decision Variable Mapping}

We build on the LLM-assisted mapping discovery strategy of \EM~\cite{zhai2025equivamap}, which uses semantic information to identify correspondences and transformations between decision-variable spaces. However, zero-shot mapping extraction can produce ambiguous or many-to-one correspondences. To reduce such errors, \SOVER{} applies a single refinement pass using a targeted prompting template (Figure~\ref{fig:prompt-variable}, Appendix). This refinement enforces mapping uniqueness and removes inconsistent assignments in which multiple source variables map to the same target variable without an explicit aggregation rule.

\subsubsection{Parameter Mapping}
Equivalent reformulations often transform not only variables but also instance parameters, such as objective coefficients, right-hand side bounds, capacities, or scaling constants. We therefore extend the semantic matching procedure beyond decision variables and extract parameter-level correspondences as well. The extracted parameter relations are passed through an analogous correction prompt (Figure~\ref{fig:prompt-parameter}, Appendix) before symbolic verification.

\subsection{Equivalence Checker}\label{subsec:EquivChecker}

Given the candidate variable and parameter mappings, \SOVER{} constructs a formal verification problem using \zt{}. Each formulation is parsed into a constraint predicate and an objective expression. Specifically, for formulation $P_A$, we write $C_A(\bx)$ for its feasible-region predicate and $O_A(\bx)$ for its objective. The target formulation $P_B$ is translated into the source coordinate system using the proposed substitution map, producing a mapped predicate $\widetilde{C}_B(\bx)$ and mapped objective $\widetilde{O}_B(\bx)$. When $P_B$ contains auxiliary/slack variables that do not enter its objective, $\widetilde{C}_B(\bx)$ denotes the existential projection of those variables, so Proposition~\ref{prop:z3} is applied to the projected feasible set rather than to the lifted coordinates themselves.

The verifier then checks two properties. First, it checks \emph{domain cross-feasibility}, namely whether both formulations define the same feasible set after substitution. Second, it checks \emph{objective-order preservation}, namely whether the two objectives induce the same weak ordering over all feasible pairs of points. This second condition is weaker than algebraic objective identity and therefore accepts valid transformations such as positive rescalings and strictly monotone objective transformations.

As shown in Figure~\ref{fig:method}, the pipeline proceeds from parsing \& symbol declaration to mapping synthesis, feasibility verification, and order-profile checking.

\begin{figure*}[!ht]
    \centering
    \includegraphics[width=1.92\columnwidth]{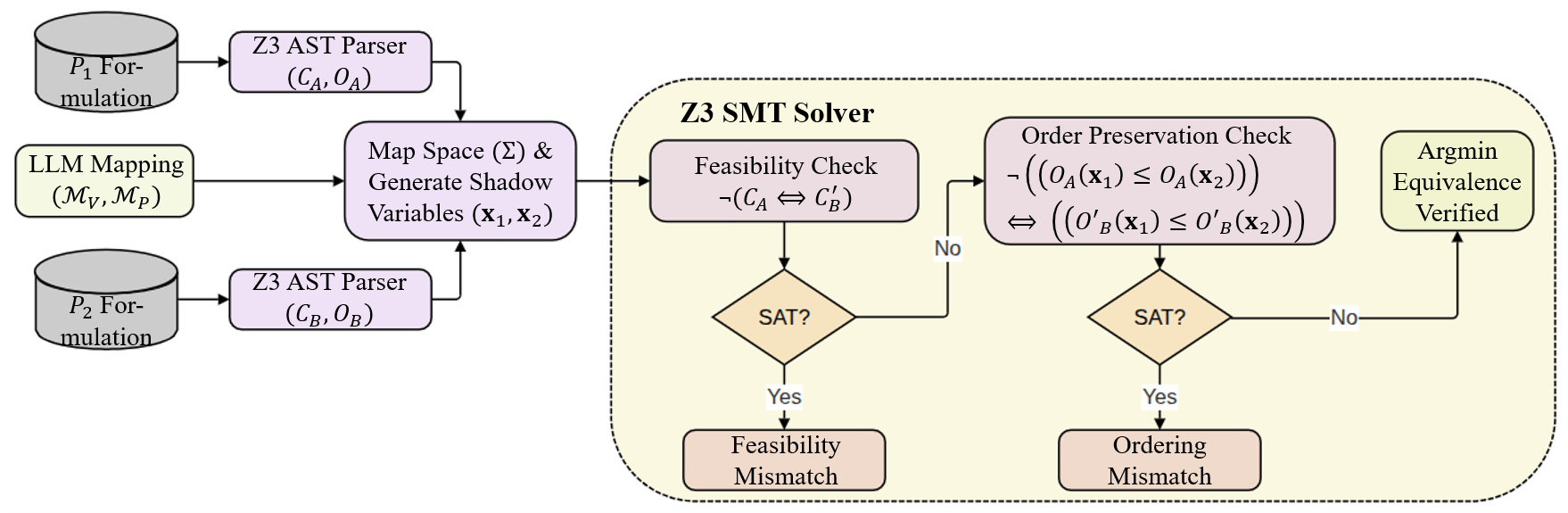}
    \caption{\SOVER{} verification pipeline. The framework parses optimization models, synthesizes variable and parameter alignments, checks domain cross-feasibility, and verifies preservation of the objective-order profile.}
    \label{fig:method}
    \vspace{-1em}
\end{figure*}

\subsection{Two-Stage SMT Verification}\label{subsec:TwoStage}

Algorithm~\ref{alg:SOVER} summarizes the core verification routine. The key idea is to express both formulations in a common symbolic environment and search for counterexamples to equivalence. If no feasibility counterexample exists and no objective-order counterexample exists, the reformulation is certified as $\arg\min$-equivalent under the proposed map.

\begin{algorithm}[!ht]
    \caption{\SOVER: SMT-Based Verification of Reformulation Equivalence}
    \label{alg:SOVER}
    \begin{algorithmic}[1]
    \Require Source problem $P_A$, target problem $P_B$, variable mapping $\mathcal{M}_V$, parameter mapping $\mathcal{M}_P$
    \Ensure Status in \{\texttt{Pass}, \texttt{Feasibility Mismatch}, \texttt{Ordering Mismatch}, \texttt{Inconclusive}\}
    
    \State $\mathcal{E} \gets \emptyset$ \Comment{Initialize symbolic environment}
    \State \Call{DeclareSymbols}{$P_A, \mathcal{E}$}
    \State \Call{DeclareSymbols}{$P_B, \mathcal{E}$}
    
    \State $C_A, O_A \gets$ \Call{ParseModel}{$P_A, \mathcal{E}$}
    \State $C_B, O_B \gets$ \Call{ParseModel}{$P_B, \mathcal{E}$}
    
    \State $\Sigma \gets$ \Call{BuildSubstitutions}{$\mathcal{M}_V \cup \mathcal{M}_P, \mathcal{E}$}
    \State $\widetilde{C}_B \gets \textsc{Substitute}(C_B,\Sigma)$
    \State $\widetilde{O}_B \gets \textsc{Substitute}(O_B,\Sigma)$
    
    \State $\Gamma \gets$ \Call{ShadowCopy}{$\mathrm{Vars}(P_A)$}
    \State $C_A^{(2)}, O_A^{(2)} \gets \textsc{Substitute}(C_A,\Gamma), \textsc{Substitute}(O_A,\Gamma)$
    \State $\widetilde{C}_B^{(2)}, \widetilde{O}_B^{(2)} \gets \textsc{Substitute}(\widetilde{C}_B,\Gamma), \textsc{Substitute}(\widetilde{O}_B,\Gamma)$
    
    \Statex
    \State \textbf{// Stage 1: domain cross-feasibility}
    \State \textbf{Assert} $\neg(C_A \leftrightarrow \widetilde{C}_B)$
    \State $r \gets$ \Call{CheckSatisfiability}{}
    \If{$r = \texttt{SAT}$}
        \State \Return \texttt{Feasibility Mismatch}
    \ElsIf{$r = \texttt{UNKNOWN}$}
        \State \Return \texttt{Inconclusive}
    \EndIf
    
    \Statex
    \State \textbf{// Stage 2: objective-order preservation}
    \State \textbf{Reset solver context}
    \State \textbf{Assert} $C_A \land \widetilde{C}_B$
    \State \textbf{Assert} $C_A^{(2)} \land \widetilde{C}_B^{(2)}$
    \State \textbf{Assert} $\neg\left((O_A \leq O_A^{(2)}) \leftrightarrow 
    (\widetilde{O}_B \leq \widetilde{O}_B^{(2)})\right)$
    \State $r \gets$ \Call{CheckSatisfiability}{}
    \If{$r = \texttt{SAT}$}
        \State \Return \texttt{Ordering Mismatch}
    \ElsIf{$r = \texttt{UNKNOWN}$}
        \State \Return \texttt{Inconclusive}
    \Else
        \State \Return \texttt{Pass}
    \EndIf
    \end{algorithmic}
\end{algorithm}

\subsubsection{Correctness Guarantee}\label{subsubsec:Correctness}

The soundness of \SOVER{} follows from the fact that the SMT queries search directly for counterexamples to feasibility \& $\argmin$ equivalence.

\begin{proposition}[Argmin Equivalence via Weak-Order Preservation]
\label{prop:z3}
Let $P_A$ and $P_B$ be two minimization problems. After applying the candidate substitution map $\Sigma$, let $C_A(\bx)$ and $\widetilde{C}_B(\bx)$ denote their feasible-region predicates in a common coordinate system, and let $O_A(\bx)$ and $\widetilde{O}_B(\bx)$ denote their corresponding objective functions. Suppose the following two properties hold:
\begin{align}
    &\forall \bx,\quad 
    C_A(\bx) \leftrightarrow \widetilde{C}_B(\bx), 
    \label{eq:feas-equiv}\\
    &\forall \bx,\by,\quad 
    C_A(\bx)\land C_A(\by)
    \Rightarrow \nonumber \\
    &\quad \left[
    O_A(\bx) \leq O_A(\by)
    \leftrightarrow
    \widetilde{O}_B(\bx) \leq \widetilde{O}_B(\by)
    \right]\!.
    \label{eq:order-equiv}
\end{align}
Then the two formulations have identical global minimizers in the common coordinate system:
\begin{equation}
    \argmin_{\bx:C_A(\bx)} O_A(\bx)
    =
    \argmin_{\bx:\widetilde{C}_B(\bx)} \widetilde{O}_B(\bx).
\end{equation}
If $\Sigma$ is induced by a bijective coordinate transformation between the original variables of $P_A$ and $P_B$, then the corresponding minimizer sets are equivalent under that transformation.
\end{proposition}

\begin{proof}
    See Appendix~\ref{prop-proof:z3} for detailed proof.
\end{proof}

\subsection{Extension to Nonlinear Programs via \dReal}\label{subsec:dReal}

The \zt{}-based verifier is well suited to linear, mixed-integer, and piecewise-linear arithmetic. However, reformulations involving transcendental functions, nonlinear coordinate transformations, or non-convex continuous constraints may fall outside the decidable fragments handled by standard SMT procedures. To support such cases, \SOVER{} extends the same verification logic using the \dReal{} $\delta$-complete decision procedure.

\subsubsection{$\delta$-Satisfiability}\label{subsubsec:deltaSatis}

For nonlinear real arithmetic with transcendental functions, exact satisfiability is generally undecidable. \dReal{} addresses this by deciding formulas up to a user-specified numerical tolerance $\delta>0$. Given a formula $\mathcal{F}$, \dReal{} returns either \texttt{UNSAT}, certifying that $\mathcal{F}$ has no solution, or $\delta$-\texttt{SAT}, indicating that a $\delta$-weakened version of the formula may be satisfiable. In our pipeline, \texttt{UNSAT} is treated as a formal certificate, while $\delta$-\texttt{SAT} is treated as a candidate counterexample or an inconclusive result requiring tolerance-aware interpretation.

A raw $\delta$-\texttt{SAT} result for a strict boundary-negation query can be produced by the $\delta$-weakening even when the exact mismatch set is empty. We therefore search only for \emph{margin-separated} nonlinear violations, using a separation margin strictly larger than $\delta$. Moreover, the $P_A\!\to P_B$ feasibility direction is checked as a forward-map range-coverage property: an $A$-feasible point is a counterexample only if no $B$-feasible point maps sufficiently close to it. This avoids falsely certifying a non-surjective transformation and does not require a globally single-valued inverse for the rejection of a bad map.

Some nonlinear reformulations introduce auxiliary variables, such as slack variables, lifted variables, or projection variables, that do not appear explicitly in the source formulation. \SOVER{} handles such variables through quantified feasibility checks, ensuring that auxiliary degrees of freedom do not create spurious mismatches between the projected feasible regions.

\subsubsection{$\epsilon$-Argmin Equivalence}\label{subsubsec:Epsilon}

For nonlinear objectives, exact order preservation may be too brittle near boundary points or flat regions. We therefore use an $\epsilon$-argmin relaxation. Instead of requiring exact equality of minimizer sets, the nonlinear verifier checks whether exact minimizers of one formulation map into the $\epsilon$-optimal region of the other formulation, and vice versa.

For a minimization problem $P$ with feasible set $\mathcal{X}$ and objective $O$, define
\begin{equation*}
    \Omega_\epsilon(P)
    \!=\!
    \left\{
    \bx\!\in\!\mathcal{X}
    \;\middle|\;
    O(\bx) \leq \inf_{\by\in\mathcal{X}} O(\by) + \epsilon
    \right\}\!.
\end{equation*}

\begin{proposition}[$\epsilon$-Argmin Preservation under $\delta$-Complete Verification]
\label{prop:dreal}
{
Let $P_A$ and $P_B$ be nonlinear continuous minimization problems whose global minima are attained, with feasible sets $\mathcal{X}_A$ and $\mathcal{X}_B$. Let $\Sigma:\mathcal{X}_B\!\to\!\mathcal{X}_A$ be a bijective coordinate transformation on the feasible sets. Suppose the two feasibility-mismatch queries are \texttt{UNSAT}, and \dReal{} also returns \texttt{UNSAT} for both margin-separated optimization-mismatch formulas asserting that an exact minimizer of one problem is mapped to a point that can be improved by more than $\epsilon$ in the other problem. Then
\begin{equation*}
\Sigma(\Omega_0(P_B))\!\subseteq\!\Omega_\epsilon(P_A),
    \
    \Sigma^{-1}(\Omega_0(P_A))\!\subseteq\!\Omega_\epsilon(P_B).
\end{equation*}
Here \texttt{UNSAT} is an exact certificate for the original mismatch formula; a $\delta$-\texttt{SAT} result is not used as an equivalence certificate.
}
\end{proposition}

\begin{proof}
    See Appendix~\ref{prop-proof:dReal} for detailed proof.
\end{proof}

\subsection{Implementation Details}\label{subsec:Implementation}

The implementation operates in four stages.\vspace{-.5em}
\begin{enumerate}[leftmargin=*]
    \item \textbf{Parsing and bounded unrolling.}
    The input formulations are programmatic \texttt{gurobipy} strings. \SOVER{} removes solver API syntax, such as \texttt{addConstr} and \texttt{setObjective}, and extracts the underlying algebraic constraints and objective expressions. For benchmark evaluation, indexed variables and matrix expressions are unrolled to a fixed proxy dimension. This certifies the resulting instantiated model; dimension-parametric verification would require an additional induction or quantified-array argument.\vspace{-.75em}
    \item \textbf{Symbol declaration and substitution.}
    Extracted variables and parameters are declared in the SMT environment using their native sorts, including Boolean, integer, and real symbols. The verified variable and parameter mappings are then converted into a substitution dictionary that expresses the target formulation in the source coordinate system.\vspace{-.75em}
    \item \textbf{Feasibility verification.}
    The verifier searches for an assignment satisfying $\neg(C_A \leftrightarrow \widetilde{C}_B)$. A satisfying assignment is a concrete feasibility counterexample. If the formula is unsatisfiable, the two formulations define the same feasible region after substitution.\vspace{-.75em}
    \item \textbf{Objective-order verification.}
    The verifier introduces a shadow copy of all source variables to represent a second feasible point. It then searches for two feasible points whose objective order differs across the two formulations. Unsatisfiability of this query certifies global weak-order preservation, and therefore $\arg\min$ equivalence by Proposition~\ref{prop:z3}.
\end{enumerate}

For each benchmark pair, \SOVER{} records the verification status, generated SMT formulas, solver outputs, counterexamples when available, and diagnostic metadata in a structured \texttt{csv} file.

\section{Experiments and Evaluation}\label{sec:Eval}

We evaluate \SOVER{} on reformulation pairs derived from the EquivaFormulation dataset\footnote{\url{huggingface.co/datasets/humainlab/EquivaFormulation}} introduced by \cite{zhai2025equivamap}. EquivaFormulation modifies problems from the NLP4LP corpus\footnote{\url{huggingface.co/datasets/udell-lab/NLP4LP}} \cite{ahmaditeshnizi2024optimus} using semantic and algebraic transformations, summarized in Table~\ref{tab:variations}. We additionally introduce \textsc{NLEquiv-150}, 150 application-grounded continuous nonlinear pairs: 100 equivalent pairs related by diverse nonlinear coordinate transformations and 50 hard negatives with subtle semantic mismatches. The positives span exponential/log, softplus, reciprocal, square/root, logistic/odds, hyperbolic, shifted-exponential, and cubic transformations; the negatives span 11 objective, constraint, equality/direction/coefficient, and range-mismatch mechanisms. Formulations, mappings, labels, metadata, and code are publicly available at \url{https://github.com/baranwa2/SOVER}. For nonlinear verification, the forward map is primary: \SOVER{} checks mapped feasibility and reverse range coverage using \dReal{} with $\delta=10^{-5}$, separation margin $10^{-3}$, and $\epsilon=10^{-3}$.

\subsection{Data Cleaning}\label{subsec:data-cleaning}

Before evaluation, we clean metadata that would otherwise corrupt SMT declarations: qualitative dimension descriptors (e.g., ``positive,'' ``continuous,'' ``non-negative'') are removed by regex, and empty/non-indexed shapes are treated as scalars unless explicit indexing is present (examples in Appendix, Figure~\ref{fig:irrelevant dimension}). We evaluate nine transformed versions of each original formulation and exclude variation \_e (added valid inequalities), since our focus is structural equivalence rather than instance-specific or solver-dependent behavior.

\subsection{Experimental Setup \& Problem Variations}\label{subsec:setup}

Each benchmark instance pairs an original optimization formulation with one transformed formulation. Table~\ref{tab:results} reports the number of correctly classified instances for each subtype. The expanded evaluation contains $2328$ pairs in Table~\ref{tab:results}: $2178$ EquivaFormulation pairs and $150$ \textsc{NLEquiv-150} pairs. %The transformations include equivalent variants, such as variable renaming, binary substitution, slack-variable introduction, linear substitution, and objective rescaling, as well as non-equivalent variants, such as unrelated formulation replacement and removal of constraints that are inactive at the observed optimum. These categories test whether an equivalence checker can distinguish semantic equivalence from superficial agreement in solver outputs. 

\begin{table*}[htbp]
\centering
\begin{tabular}{lcccccccc}
\hline
\textbf{Subtype} & \textbf{Total} & \textbf{\SOVER} & \textbf{\EM} & \textbf{WLT} & \textbf{Naive} & \textbf{Gemini} & \textbf{Canonical}  \\ 
 & & \textbf{(Ours)} & & & \textbf{Prompt} & \textbf{Aided} & \textbf{Accuracy}  \\ 
\hline
\_c & 243 & \textbf{243} & 234 & 243 & 152 & 226 & 224  \\
\_d & 234 & \textbf{234} & 207 & 63  & 56  & 79  & 58   \\
\_f & 243 & \textbf{243} & 233 & 0   & 84  & 110 & 0    \\
\_g & 243 & \textbf{241} & 234 & 51  & 57  & 60  & 42   \\
\_h & 243 & \textbf{243} & 186 & 51  & 1   & 6   & 0    \\ 
\_i & 243 & \textbf{242} & 166 & 243 & 82  & 157 & 0    \\
\_j & 243 & \textbf{243} & 241 & 232 & 239 & 238 & 243  \\
\_k & 243 & \textbf{243} & 237 & 243 & 243 & 242 & 243  \\
\_l & 243 & \textbf{241} & 197 & 243 & 243 & 240 & 243  \\
\_m$^{+}$ & 100 & \textbf{99} & -- & -- & -- & -- & --  \\
\_m$^{-}$ & 50 & \textbf{50} & \-- & -- & -- & -- & --  \\
\hline
\end{tabular}
\caption{Comparative performance across problem-transformation subtypes. Entries report the number of correctly classified instances. \textsc{NLEquiv-150} comprises $\_m^{+}$ (100 equivalent) and $\_m^{-}$ (50 hard non-equivalent) pairs. Counts are end-to-end with LLM-extracted mappings; the sole $\_m^{+}$ miss is an incomplete map. Baselines target the original EquivaFormulation setting and are omitted for these new pairs.}
\label{tab:results}
\end{table*}

\subsection{Baselines}\label{subsec:baselines}

We compare against five baselines:\vspace{-.5em}
\begin{itemize}[leftmargin=*]
    \item \textbf{\EM{} \cite{zhai2025equivamap}:} LLM-derived variable mappings followed by feasibility/optimality checks.\vspace{-1em}
    \item \textbf{Naive LLM Prompting:} Zero-shot binary equivalence classification \cite{zhai2025equivamap} without explicit mappings or solvers.\vspace{-1em}
    \item \textbf{Gemini-CoT:} Chain-of-Thought classification \cite{team2023gemini,wei2022chain} with intermediate structural reasoning (prompt in Appendix, Figure~\ref{fig:prompt-gemini}).\vspace{-1em}
    \item \textbf{Weisfeiler-Lehman Graph Test (WLT) \cite{douglas2011weisfeiler}:} Bipartite formulation graphs compared by iterative color refinement.\vspace{-1em}
    \item \textbf{Canonical Accuracy:} Exact character-level declaration matching.
\end{itemize}

\subsection{Results and Analysis}\label{subsec:results-analysis}

On EquivaFormulation, \SOVER{} correctly classifies $2173/2178$ MILP pairs ($99.77\%$): $1446/1449$ equivalent and $727/729$ non-equivalent. It obtains $242/243$ on objective rescaling (\_i), where canonical matching obtains $0/243$, and perfect accuracy on objective-to-constraint (\_f) and linear-substitution (\_h) variants. For the challenging non-equivalent \_l subtype, where inactive constraints are removed without necessarily changing an observed optimum, explicit feasibility checking yields $241/243$. These results show the benefit of feasible-region and objective-order reasoning over surface or single-optimum agreement.

Our reproduction of \EM{} with its released code and multiple frontier LLM backends obtains $1935/2178$ ($88.84\%$), including $186/243$ on \_h and $166/243$ on \_i. The observed errors are concentrated in cases requiring precise mappings: inaccurate, incomplete, or non-unique single-shot mappings propagate to the final decision, whereas \SOVER{} formally checks the proposed alignment through SMT feasibility and ordering obligations.
\begin{table*}[htbp]
\centering
\small
\resizebox{0.9\textwidth}{!}{%
\begin{tabular}{lcccccc}
\hline
\textbf{Metric} & \textbf{SOVER} & \textbf{EquivaMap} & \textbf{WLT} &
\textbf{Naive} & \textbf{Gemini} & \textbf{Canonical} \\
& \textbf{(End-to-End)} & & & \textbf{Prompt} &
\textbf{Aided} & \textbf{Accuracy} \\
\hline
\textbf{End-to-end time (s)} & 7.07 & 2.60 & 1.01 & 1.17 & 2.18 & 7.04 \\
\textbf{Formal verifier only (s)} & 0.03 (Z3) & -- & -- & -- & -- & -- \\
\hline
\end{tabular}%
}
\caption{Average approximate evaluation time per instance. For \SOVER{}, 7.07\,s is the complete mapping-plus-verification pipeline, whereas the \zt{} verification stage alone averages 0.03\,s.}
\label{tab:time_results}
\end{table*}
In order to prove robustness of our framework, we have also evaluated \SOVER{} directly on \textbf{FormulationBench} \cite{robbins2026flare}, beyond the EquivaFormulation data used in the original submission. We conducted \textbf{116 formulation-level equivalence tests spanning all 20 base problems}: 5 from EquivaFormulation \cite{zhai2025equivamap}, 7 from EvoCut\cite{yazdani2026evocutstrengtheningintegerprograms}, and 8 from \cite{ferchtandiker2025finding}. \SOVER{} correctly classifies \textbf{111/116 cases (95.69\%)}. The latest version of \textbf{FormulationBench} contains 109 formulations, but our experimentation is based on the version as of 12 July 2026. The detailed results are highlighted in Table~\ref{tab:formulationbench_results}.

\begin{table}[!htbp]
\centering

\begin{tabular}{lccc}
\toprule
\textbf{Variation} & \textbf{Total} & \textbf{Mismatches} & \textbf{Correct} \\
\midrule
\texttt{\_a} & 20 & 0 & 20 \\
\texttt{\_b} & 20 & 1 & 19 \\
\texttt{\_c} & 12 & 1 & 11 \\
\texttt{\_d} & 12 & 1 & 11 \\
\texttt{\_e} & 10 & 1 & 9 \\
\texttt{\_f} & 9  & 1 & 8 \\
\texttt{\_g} & 9  & 0 & 9 \\
\texttt{\_h} & 9  & 0 & 9 \\
\texttt{\_i} & 9  & 0 & 9 \\
\texttt{\_j} & 6  & 0 & 6 \\
\midrule
\textbf{Total} & \textbf{116} & \textbf{5} & \textbf{111} \\
\bottomrule
\end{tabular}
\caption{Performance of \SOVER{} across different formulation variations in FormulationBench.}
\label{tab:formulationbench_results}
\end{table}

\textbf{Nonlinear benchmark analysis.} On \textsc{NLEquiv-150}, the LLM recovers 99/100 intended positive forward maps and \SOVER{} certifies those same 99; the only miss is an incomplete extracted map. \SOVER{} rejects all 50 hard negatives. These include five cases each of eight mismatch types (constraint tightening/addition/omission, functional mismatch, objective linear perturbation, equality mismatch, objective risk-target shift, coefficient mismatch), four direction mismatches, and three each of objective-target and nonlinear-range mismatches. Of three imperfect negative mapping records, one has an incorrect forward map and two have correct non-injective $\tanh^2(\cdot)$ maps with single-branch inverses; forward-map range checking still rejects all three. Overall accuracy is $149/150=99.33\%$ on \textsc{NLEquiv-150} and $2322/2328=99.74\%$ including EquivaFormulation.

\textbf{Runtime decomposition.} We do not claim end-to-end speed superiority: \SOVER{} averages 7.07\,s per EquivaFormulation pair versus 2.60\,s for \EM{}. Yet \zt{} verification itself averages only 0.03\,s (about 0.4\%); almost all remaining time is LLM-assisted mapping/refinement. Thus, robustness is obtained through stricter upstream semantic alignment, while certification is inexpensive: once a candidate map is available, \SOVER{} verifies feasibility/order obligations directly rather than independently solving both formulations and comparing optima. The advantage is therefore robust, low-cost verification—not lower present end-to-end latency.

\section{Conclusion and Future Work}\label{sec:Conclusion}
We presented \SOVER, an LLM-assisted SMT framework for verifying optimization reformulations. Rather than comparing solver outputs empirically, \SOVER{} encodes correctness as logical obligations over symbolic domains. For \zt{} cases, cross-feasibility and objective-order preservation certify exact $\argmin$ equivalence; for nonlinear \dReal{} cases, margin-separated \texttt{UNSAT} queries certify the stated $\epsilon$-argmin guarantee. Experimentation using EquivaFormulation shows that \SOVER{} substantially outperforms prompting, syntactic, graph-based, and mapping-based baselines, particularly in cases where equivalence cannot be inferred from optimal values alone.

Several directions remain open. Improving LLM-based mapping synthesis could reduce inconclusive cases, while extending the verifier beyond bounded instances to dimension-parametric formulations would provide stronger guarantees for problem families. Although the \dReal{} extension supports nonlinear continuous reformulations via $\delta$-satisfiability, tighter tolerance-aware certificates for non-convex models remain important. Finally, integrating \SOVER{} into automated modeling pipelines could provide end-to-end safeguards for LLM-generated formulations before deployment in high-stakes settings.

% We presented \SOVER, an LLM-assisted SMT-based framework for verifying optimization reformulations. Instead of relying on empirical agreement between solver outputs, \SOVER{} encodes reformulation correctness as logical obligations over symbolic domains. By checking domain cross-feasibility and global objective-order preservation, the framework verifies $\argmin$ equivalence while avoiding common failure modes from inactive constraints, scaled objectives, symbolic parameters, and heterogeneous representations. Our experiments on EquivaFormulation demonstrate that \SOVER{} is significantly more robust than prompting, syntactic, graph, and mapping baselines, especially where equivalence cannot be determined from optimal objective values alone.

% Several directions remain open. First, while the final equivalence decision is SMT-verified, improving the underlying LLM mapping synthesis would reduce inconclusive cases. Second, extending the verifier from bounded benchmark instances to dimension-parametric formulations would enable stronger guarantees for problem families. Third, while the \dReal{} extension supports nonlinear continuous reformulations via $\delta$-satisfiability, tighter tolerance-aware certificates for non-convex models remain important. Finally, integrating \SOVER{} into automated modeling pipelines could provide end-to-end safeguards for LLM-generated formulations before deployment in high-stakes systems.

\clearpage

\section*{Limitations}
Although our SMT-driven framework demonstrates high reliability in verifying structural equivalence, it is subject to several operational limitations. Primarily, the system is bottlenecked by its dependence on the upstream LLM mapping phase; because the deterministic solver relies on these generated alignments, any failure by the language model to extract a precise, unique mapping—particularly in highly convoluted or obfuscated reformulations—directly results in a verification failure. Secondly, relying on an SMT solver like \zt{} introduces inherent scalability challenges. While highly effective for standard formulations, symbolic verification can experience exponential computational overhead when applied to massive, industrial-scale models with thousands of integer variables and constraints. Finally, as observed in specific edge cases involving slack variables, our current feasibility parser relies on a fixed unrolling proxy depth. This architectural choice occasionally restricts the engine's ability to resolve deeply nested inequalities or complex nonlinear bounds, presenting a potential area for refinement in future iterations of the pipeline. 
\section*{Ethical Considerations}
The deployment of automated verification frameworks for optimization models carries several ethical implications. A primary concern is automation bias in high-stakes domains (e.g., healthcare logistics or resource allocation); practitioners might over-rely on automated approvals, potentially leading to real-world mismanagement if a flawed model bypasses detection due to upstream extraction errors. Additionally, while the SMT solver provides deterministic proofs, the LLM-based mapping phase remains an opaque "black box," complicating the full auditability of the verification pipeline. Finally, the environmental and computational costs of querying large language models in tandem with resource-intensive symbolic solvers must be carefully weighed in future deployments.
\bibliography{custom}

\begin{thebibliography}{38}
\providecommand{\natexlab}[1]{#1}

\bibitem[{Aaronson(2005)}]{aaronson2005guest}
Scott Aaronson. 2005.
\newblock {NP-Complete Problems and Physical Reality}.
\newblock \emph{ACM Sigact News}, 36(1):30--52.

\bibitem[{AhmadiTeshnizi et~al.(2024)AhmadiTeshnizi, Gao, and Udell}]{ahmaditeshnizi2024optimus}
Ali AhmadiTeshnizi, Wenzhi Gao, and Madeleine Udell. 2024.
\newblock {Optimus: Scalable Optimization Modeling with (mi) LP Solvers and Large Language Models}.
\newblock \emph{arXiv preprint arXiv:2402.10172}.

\bibitem[{Brown et~al.(2020)Brown, Mann, Ryder, Subbiah, Kaplan, Dhariwal, Neelakantan, Shyam, Sastry, Askell et~al.}]{brown2020language}
Tom Brown, Benjamin Mann, Nick Ryder, Melanie Subbiah, Jared~D Kaplan, Prafulla Dhariwal, Arvind Neelakantan, Pranav Shyam, Girish Sastry, Amanda Askell, and 1 others. 2020.
\newblock {Language Models are Few-shot Learners}.
\newblock \emph{Advances in neural information processing systems}, 33:1877--1901.

\bibitem[{Chen et~al.(2024)Chen, Constante-Flores, and Li}]{chen2024diagnosing}
Hao Chen, Gonzalo~E Constante-Flores, and Can Li. 2024.
\newblock {Diagnosing Infeasible Optimization Problems using Large Language Models}.
\newblock \emph{INFOR: Information Systems and Operational Research}, 62(4):573--587.

\bibitem[{Chen et~al.(2021)Chen, Tworek, Jun, Yuan, Pinto, Kaplan, Edwards, Burda, Joseph, Brockman et~al.}]{chen2021evaluating}
Mark Chen, Jerry Tworek, Heewoo Jun, Qiming Yuan, Henrique Ponde De~Oliveira Pinto, Jared Kaplan, Harri Edwards, Yuri Burda, Nicholas Joseph, Greg Brockman, and 1 others. 2021.
\newblock {Evaluating Large Language Models Trained on Code}.
\newblock \emph{arXiv preprint arXiv:2107.03374}.

\bibitem[{Chen et~al.(2026)Chen, Xia, Shao, Ge, and Ye}]{chen2026solver}
Yitian Chen, Jingfan Xia, Siyu Shao, Dongdong Ge, and Yinyu Ye. 2026.
\newblock {Solver-informed RL: Grounding Large Language Models for Authentic Optimization Modeling}.
\newblock \emph{Advances in Neural Information Processing Systems}, 38:106027--106069.

\bibitem[{Christofides et~al.(1981)Christofides, Mingozzi, and Toth}]{DBLP:journals/mp/ChristofidesMT81}
Nicos Christofides, Aristide Mingozzi, and Paolo Toth. 1981.
\newblock \href {https://doi.org/10.1007/BF01589353} {{Exact Algorithms for the Vehicle Routing Problem, based on Spanning Tree and Shortest Path Relaxations}}.
\newblock \emph{Math. Program.}, 20(1):255--282.

\bibitem[{Douglas(2011)}]{douglas2011weisfeiler}
Brendan~L Douglas. 2011.
\newblock {The Weisfeiler-Lehman Method and Graph Isomorphism Testing}.
\newblock \emph{arXiv preprint arXiv:1101.5211}.

\bibitem[{Ferchtandiker et~al.(2025)Ferchtandiker, den Hertog, Udell, and Wasserkrug}]{ferchtandiker2025finding}
Nathan Ferchtandiker, Dick den Hertog, Madeleine Udell, and Segev Wasserkrug. 2025.
\newblock \href {https://github.com/nathan-ferchtandiker/LLMs-For-Optimization-Reformulations} {Finding efficient milo formulations with llms}.
\newblock Working paper.

\bibitem[{Grattafiori et~al.(2024)Grattafiori, Dubey, Jauhri, Pandey, Kadian, Al-Dahle, Letman, Mathur, Schelten, Vaughan et~al.}]{grattafiori2024llama}
Aaron Grattafiori, Abhimanyu Dubey, Abhinav Jauhri, Abhinav Pandey, Abhishek Kadian, Ahmad Al-Dahle, Aiesha Letman, Akhil Mathur, Alan Schelten, Alex Vaughan, and 1 others. 2024.
\newblock {The Llama 3 Herd of Models}.
\newblock \emph{arXiv preprint arXiv:2407.21783}.

\bibitem[{Herrera-Poyatos et~al.(2025)Herrera-Poyatos, Pel{\'a}ez-Gonz{\'a}lez, Zuheros, Herrera-Poyatos, Tejedor, Herrera, and Montes}]{herrera2025overview}
David Herrera-Poyatos, Carlos Pel{\'a}ez-Gonz{\'a}lez, Cristina Zuheros, Andr{\'e}s Herrera-Poyatos, Virilo Tejedor, Francisco Herrera, and Rosana Montes. 2025.
\newblock {An Overview of Model Uncertainty and Variability in LLM-based Sentiment Analysis: Challenges, Mitigation Strategies, and the Role of Explainability}.
\newblock \emph{Frontiers in Artificial Intelligence}, 8:1609097.

\bibitem[{Huang et~al.(2025)Huang, Shen, Hu, Gao, and Wang}]{huang2025llms}
Xuhan Huang, Qingning Shen, Yan Hu, Anningzhe Gao, and Benyou Wang. 2025.
\newblock {LLMs for Mathematical Modeling: Towards Bridging the Gap Between Natural and Mathematical Languages}.
\newblock In \emph{Findings of the Association for Computational Linguistics: NAACL 2025}, pages 2678--2710.

\bibitem[{Karp(2009)}]{karp2009reducibility}
Richard~M Karp. 2009.
\newblock {Reducibility Among Combinatorial Problems}.
\newblock In \emph{50 Years of Integer Programming 1958-2008: from the Early Years to the State-of-the-Art}, pages 219--241. Springer.

\bibitem[{Le(2024)}]{le2024survey}
Phuong Le. 2024.
\newblock {A Survey on Combinatorial Optimization}.
\newblock \emph{arXiv preprint arXiv:2409.00075}.

\bibitem[{Li et~al.(2025{\natexlab{a}})Li, Kulkarni, Menache, Wu, and Li}]{li2025towards}
Sirui Li, Janardhan Kulkarni, Ishai Menache, Cathy Wu, and Beibin Li. 2025{\natexlab{a}}.
\newblock {Towards Foundation Models for Mixed Integer Linear Programming}.
\newblock In \emph{International Conference on Learning Representations}, volume 2025, pages 88590--88638.

\bibitem[{Li et~al.(2025{\natexlab{b}})Li, Jin, Hong, Lu, and Wang}]{li2025optimization}
Wenhao Li, Bo~Jin, Mingyi Hong, Changhong Lu, and Xiangfeng Wang. 2025{\natexlab{b}}.
\newblock {Optimization Problem Solving Can Transition to Evolutionary Agentic Workflows}.
\newblock \emph{arXiv preprint arXiv:2505.04354}.

\bibitem[{Liu et~al.(2022)Liu, Lu, Abbasi, Li, Mohammadi, and Kolouri}]{liu2022teachingnetworkssolveoptimization}
Xinran Liu, Yuzhe Lu, Ali Abbasi, Meiyi Li, Javad Mohammadi, and Soheil Kolouri. 2022.
\newblock \href {https://arxiv.org/abs/2202.04104} {{Teaching Networks to Solve Optimization Problems}}.
\newblock \emph{Preprint}, arXiv:2202.04104.

\bibitem[{Ma et~al.(2026)Ma, Dai, Yuan, Li, Luo, Wang, Liu, Sha, and Sui}]{ma2026large}
Jingyuan Ma, Damai Dai, Zihang Yuan, Rui Li, Weilin Luo, Bin Wang, Qun Liu, Lei Sha, and Zhifang Sui. 2026.
\newblock {Large Language Models Struggle with Unreasonability in Math Problems}.
\newblock In \emph{Proceedings of the AAAI Conference on Artificial Intelligence}, volume~40, pages 32428--32436.

\bibitem[{Mahmoud et~al.(2024)Mahmoud, Mohammed, Ayman, Medhat, Selim, Zayed, Yousef, and Elaraby}]{mahmoud2024formal}
Amira~T Mahmoud, Ahmad~A Mohammed, Mahitap Ayman, Walaa Medhat, Sahar Selim, Hala Zayed, Ahmed~H Yousef, and Nahla Elaraby. 2024.
\newblock {Formal Verification of Code Conversion: A Comprehensive Survey}.
\newblock \emph{Technologies}, 12(12):244.

\bibitem[{Paulson(2006)}]{paulson2006defining}
Lawrence~C Paulson. 2006.
\newblock {Defining Functions on Equivalence Classes}.
\newblock \emph{ACM Transactions on Computational Logic (TOCL)}, 7(4):658--675.

\bibitem[{Pei et~al.(2025)Pei, Du, and Jin}]{pei-etal-2025-fover}
Yu~Pei, Yongping Du, and Xingnan Jin. 2025.
\newblock \href {https://doi.org/10.1162/tacl.a.41} {{ "{F}o{V}er: First-Order Logic Verification for Natural Language Reasoning"}}.
\newblock \emph{Transactions of the Association for Computational Linguistics}, 13:1340--1359.

\bibitem[{Pop et~al.(2024)Pop, Cosma, Sabo, and Sitar}]{pop2024comprehensive}
Petric{\u{a}}~C Pop, Ovidiu Cosma, Cosmin Sabo, and Corina~Pop Sitar. 2024.
\newblock A comprehensive survey on the generalized traveling salesman problem.
\newblock \emph{European Journal of Operational Research}, 314(3):819--835.

\bibitem[{Qin et~al.(2024)Qin, Liang, Ye, Zhu, Yan, Lu, Lin, Cong, Tang, Qian et~al.}]{qin2024toolllm}
Yujia Qin, Shihao Liang, Yining Ye, Kunlun Zhu, Lan Yan, Yaxi Lu, Yankai Lin, Xin Cong, Xiangru Tang, Bill Qian, and 1 others. 2024.
\newblock {Toolllm: Facilitating Large Language Models to Master 16000+ Real-World Apis}.
\newblock In \emph{International Conference on Learning Representations}, volume 2024, pages 9695--9717.

\bibitem[{Robbins et~al.(2026)Robbins, Lawless, Udell, and Vitercik}]{robbins2026flare}
Henry Robbins, Connor Lawless, Madeleine Udell, and Ellen Vitercik. 2026.
\newblock \href {https://flare.henryrobbins.com} {Flare: Verifying milp reformulations with llm-based theorem proving}.
\newblock Working paper.

\bibitem[{Shi et~al.(2024)Shi, Shen, Huang, Li, Leng, Jin, Liu, Wu, Guo, Yu et~al.}]{shi2024large}
Dan Shi, Tianhao Shen, Yufei Huang, Zhigen Li, Yongqi Leng, Renren Jin, Chuang Liu, Xinwei Wu, Zishan Guo, Linhao Yu, and 1 others. 2024.
\newblock {Large Language Model Safety: A Holistic Survey}.
\newblock \emph{arXiv preprint arXiv:2412.17686}.

\bibitem[{Singh(2012)}]{singh2012overview}
Ajay Singh. 2012.
\newblock {An Overview of the Optimization Modelling Applications}.
\newblock \emph{Journal of Hydrology}, 466:167--182.

\bibitem[{Sun et~al.(2019)Sun, Cao, Zhu, and Zhao}]{sun2019survey}
Shiliang Sun, Zehui Cao, Han Zhu, and Jing Zhao. 2019.
\newblock {A Survey of Optimization Methods from a Machine Learning Perspective}.
\newblock \emph{IEEE Transactions on Cybernetics}, 50(8):3668--3681.

\bibitem[{Team et~al.(2023)Team, Anil, Borgeaud, Alayrac, Yu, Soricut, Schalkwyk, Dai, Hauth, Millican et~al.}]{team2023gemini}
Gemini Team, Rohan Anil, Sebastian Borgeaud, Jean-Baptiste Alayrac, Jiahui Yu, Radu Soricut, Johan Schalkwyk, Andrew~M Dai, Anja Hauth, Katie Millican, and 1 others. 2023.
\newblock {Gemini: A Family of Highly Capable Multimodal Models}.
\newblock \emph{arXiv preprint arXiv:2312.11805}.

\bibitem[{Wang and Li(2025)}]{wang2025large}
Yang Wang and Kai Li. 2025.
\newblock {Large Language Models in Operations Research: Methods, Applications, and Challenges}.
\newblock \emph{arXiv preprint arXiv:2509.18180}.

\bibitem[{Wei et~al.(2022)Wei, Wang, Schuurmans, Bosma, Xia, Chi, Le, Zhou et~al.}]{wei2022chain}
Jason Wei, Xuezhi Wang, Dale Schuurmans, Maarten Bosma, Fei Xia, Ed~Chi, Quoc~V Le, Denny Zhou, and 1 others. 2022.
\newblock {Chain-Of-Thought Prompting Elicits Reasoning in Large Language Models}.
\newblock \emph{Advances in neural information processing systems}, 35:24824--24837.

\bibitem[{Xiao et~al.(2025)Xiao, Xie, Xu, Guan, Zhu, Han, Fu, Yu, Wu, Shi et~al.}]{xiao2025survey}
Ziyang Xiao, Jingrong Xie, Lilin Xu, Shisi Guan, Jingyan Zhu, Xiongwei Han, Xiaojin Fu, WingYin Yu, Han Wu, Wei Shi, and 1 others. 2025.
\newblock {A Survey of Optimization Modeling meets LLMs: Progress and Future Directions}.
\newblock \emph{arXiv preprint arXiv:2508.10047}.

\bibitem[{Yang et~al.(2024)Yang, Wang, Lu, Liu, Le, Zhou, and Chen}]{yang2024large}
Chengrun Yang, Xuezhi Wang, Yifeng Lu, Hanxiao Liu, Quoc~V Le, Denny Zhou, and Xinyun Chen. 2024.
\newblock {Large Language Models as Optimizers}.
\newblock In \emph{International Conference on Learning Representations}, volume 2024, pages 12028--12068.

\bibitem[{Yao et~al.(2025)Yao, Sun, and Xue}]{yao2025fact}
Jiayi Yao, Haibo Sun, and Nianwen Xue. 2025.
\newblock {Fact-checking AI-generated News Reports: Can LLMs Catch Their Own Lies?}
\newblock \emph{arXiv preprint arXiv:2503.18293}.

\bibitem[{Yazdani et~al.(2026)Yazdani, Mostajabdaveh, Aref, and Zhou}]{yazdani2026evocutstrengtheningintegerprograms}
Milad Yazdani, Mahdi Mostajabdaveh, Samin Aref, and Zirui Zhou. 2026.
\newblock \href {https://arxiv.org/abs/2508.11850} {Evocut: Strengthening integer programs via evolution-guided language models}.
\newblock \emph{Preprint}, arXiv:2508.11850.

\bibitem[{Zhai et~al.(2025)Zhai, Lawless, Vitercik, and Leqi}]{zhai2025equivamap}
Haotian Zhai, Connor Lawless, Ellen Vitercik, and Liu Leqi. 2025.
\newblock {EquivaMap: Leveraging LLMs for Automatic Equivalence Checking of Optimization Formulations}.
\newblock \emph{arXiv preprint arXiv:2502.14760}.

\bibitem[{Zhang et~al.(2025)Zhang, Cheng, Yi, and Tan}]{zhang2025systematic}
Yisong Zhang, Ran Cheng, Guoxing Yi, and Kay~Chen Tan. 2025.
\newblock {A Systematic Survey on Large Language Models for Evolutionary Optimization: From Modeling to Solving}.
\newblock \emph{arXiv preprint arXiv:2509.08269}.

\bibitem[{Zhou et~al.(2025)Zhou, Xu, Lin, and Ge}]{zhou2025steporlm}
Chenyu Zhou, Tianyi Xu, Jianghao Lin, and Dongdong Ge. 2025.
\newblock {Steporlm: A Self-evolving Framework with Generative Process Supervision for Operations Research Language Models}.
\newblock \emph{arXiv preprint arXiv:2509.22558}.

\bibitem[{Zhou and Zhang(2025)}]{zhou2025step}
Kuo Zhou and Lu~Zhang. 2025.
\newblock {Step-wise Formal Verification for LLM-based Mathematical Problem Solving}.
\newblock \emph{arXiv preprint arXiv:2505.20869}.

\end{thebibliography}

\clearpage
\appendix

\section{Appendix}
\label{sec:appendix}

\section*{GenAI Usage Disclosure}
We used OpenAI's ChatGPT only to assist with limited language refinement of the manuscript. We did not use GenAI tools for problem formulation, algorithm design, theoretical analysis, proof development, code generation, data collection, data preprocessing, experimental design, result generation, result analysis, or scientific interpretation. All technical content, experimental results, scientific claims, and conclusions are entirely the authors' own and were verified by the authors.

\subsection{Inference Hyperparameters and Evaluation Libraries}
%\begin{description}
    \textbf{LLM Generation Pipeline:}
    \begin{itemize}
        \item \textbf{Orchestration Framework:} All generation queries are orchestrated via the \texttt{liteLLM} abstraction framework to interface uniformly with the GPT-5.4-mini deployment endpoint.
        \item \textbf{Hyperparameter Configuration:} To ensure deterministic, reproducible structural alignments, the decoding temperature is strictly set to $\tau = 0.0$. The \texttt{top\_p} parameter is fixed at $1.0$, while both \texttt{presence\_penalty} and \texttt{frequency\_penalty} are maintained at $0.0$ to eliminate stochastic variance in structured JSON production.
        \item \textbf{Prompt-Level Context Rules:} Generative behavior is bounded by a strict structural bijection constraint, preventing the mapping of multiple distinct Problem~1 source variables to a singular Problem~2 target variable. Furthermore, a ``scalar-by-default'' logic is mandated, requiring explicit array indexing evidence before declaring any variable a vector.
    \end{itemize}
    
    \noindent\textbf{Formal Verification Engine:}
    \begin{itemize}
        \item \textbf{Model Parsing:} The \texttt{gurobipy} API is utilized to extract and sanitize the underlying mixed-integer and linear programming models, translating constraint and objective functions into programmatic Python representations.
        \item \textbf{Symbolic Evaluation:} Core formal equivalence proofs are evaluated using the \zt{} Theorem Prover via the \texttt{\zt-solver} Python library, serving as the deterministic Satisfiability Modulo Theories (SMT) backend.
        \item \textbf{Structural Heuristics:} To avoid proof corruption from superficial encoding differences, slack and auxiliary variables are automatically isolated using regular expressions. These are parameterized as existentially quantified entities ($\exists\,\text{slack}$), preventing them from disrupting the primary structural identity check.
        \item \textbf{Computational Bounding:} To manage undecidability boundaries and combinatorial explosion during SMT checking, each \zt{} instance is constrained by a strict execution timeout of $10,000$\,ms. This guarantees clean classification into empirical results (\textit{Pass}, \textit{Feasibility Mismatch}, \textit{Objective Mismatch}, or \textit{Runtime Error}) without unbounded hanging.
    \end{itemize}
% \end{description}

\subsection{Proof of Proposition~\ref{prop:z3}}\label{prop-proof:z3}
\begin{proof}
    By Eq.~\eqref{eq:feas-equiv}, both formulations define the same feasible set after substitution. Let this common feasible set be $\mathcal{X}$.
    
    We first show that every minimizer of $P_A$ is also a minimizer of the mapped target problem. Let $\bx^\star \in \arg\min_{\bx\in\mathcal{X}} O_A(\bx)$. Then, for every $\bx\in\mathcal{X}$,
    \[
        O_A(\bx^\star) \leq O_A(\bx).
    \]
    Applying Eq.~\eqref{eq:order-equiv} with $\bx^\star$ and $\bx$ gives
    \[
        \widetilde{O}_B(\bx^\star) \leq \widetilde{O}_B(\bx),
    \]
    for every $\bx\in\mathcal{X}$. Hence $\bx^\star$ is also a global minimizer of the mapped target problem.
    
    The reverse inclusion follows identically. If $\bx^\star$ minimizes $\widetilde{O}_B$ over $\mathcal{X}$, then Eq.~\eqref{eq:order-equiv} implies that $O_A(\bx^\star)\leq O_A(\bx)$ for every feasible $\bx$. Thus $\bx^\star$ also minimizes $O_A$. Therefore, the two argmin sets are identical.
\end{proof}

\subsection{Proof of Proposition~\ref{prop:dreal}}\label{prop-proof:dReal}
\begin{proof}
{
We prove the first containment; the second follows by applying the same argument to the bijection $\Sigma^{-1}$. Suppose, for contradiction, that there exists $\bx^\star\in\Omega_0(P_B)$ such that $\bx'=\Sigma(\bx^\star)\notin\Omega_\epsilon(P_A)$. Feasibility preservation ensures $\bx'\in\mathcal{X}_A$. By the definition of $\Omega_\epsilon(P_A)$,
\[
    O_A(\bx') >
    \inf_{\by\in\mathcal{X}_A} O_A(\by)+\epsilon .
\]
Because the global minimum of $P_A$ is attained, choose
$\by^\star\in\Omega_0(P_A)$. Then
\[
    O_A(\by^\star)
    =
    \inf_{\by\in\mathcal{X}_A}O_A(\by)
    <
    O_A(\bx')-\epsilon .
\]
Consequently, $\bx^\star$ is an exact minimizer of $P_B$ whose mapped point in $P_A$ admits a feasible competitor improving the objective by more than $\epsilon$. Hence the corresponding optimization-mismatch formula is satisfiable. A sound \dReal{} call therefore cannot return \texttt{UNSAT} for that formula, contradicting the hypothesis. Thus
$\Sigma(\Omega_0(P_B))\subseteq\Omega_\epsilon(P_A)$.
The reverse containment follows symmetrically using $\Sigma^{-1}$.

The argument relies only on the soundness of the \texttt{UNSAT} answer for the original mismatch formula. A $\delta$-\texttt{SAT} answer, which may arise from the $\delta$-weakening near a shared boundary, is intentionally not used to certify equivalence.
}
\end{proof}

\subsection{Prompts}
\label{prompts}
The secondary prompts for variable and parameter mapping generation, prompt for non-linear mapping generation, the Gemini-aided Chain-of-Thought (CoT) prompt for generation, and examples of dataset instances with dimensional issues are provided in the subsequent parts of this appendix. The complete \textsc{NLEquiv-150} release---including source/reformulated nonlinear formulations, ground-truth and LLM-extracted mappings, labels, mismatch metadata, and verification code---is available at \url{https://github.com/baranwa2/SOVER}.

\onecolumn
\begin{figure*}
\centering
\tcbset{}
\begin{tcolorbox}[width=\textwidth]
\begin{Verbatim}
You are an expert AI assistant mapping decision variables between two 
optimization problems. Your goal is to find the equivalent variables in 
Problem 2 for EVERY variable in Problem 1.
**Problem 1 (Reference):** -Variables: {json.dumps(vars1, indent=2)}
-Constraints:{json.dumps([c.get('formulation') for c in constraints1],indent=2)}
-Objective: {json.dumps(objective1.get('formulation',''),indent=2)}
**Problem 2 (Target):** -Variables:{json.dumps(vars2,indent=2)}
-Constraints:{json.dumps([c.get('formulation') for c in constraints2],indent=2)}
- Objective: {json.dumps(objective2.get('formulation', ''), indent=2)}
**Previous Incorrect Mapping (TO BE CORRECTED):** {json.dumps(old_mappings,
indent=2)}
**STRICT RULES FOR VARIABLE MAPPING:** 1. **MANDATORY CORRECTION (CRITICAL):**
   - The "Previous Incorrect Mapping" provided above is MATHEMATICALLY INVALID 
   and FAILED equivalence tests. 
   - You are strictly forbidden from returning the exact same mapping. You MUST 
   identify the specific error internally (e.g., mismatched variable, inverted
   coefficient like using 10 instead of 0.1) and change it in your final output. 
   - If you return the `old_mappings` unchanged, you have failed your core
   directive.
2. **CORE MAPPING LOGIC:**
   - Every Problem 1 variable MUST be mapped. Do NOT map to "none".
   - UNIQUE MAPPING: Do not map multiple different Problem 1 variables to the 
   exact same Problem 2 variable (e.g., if x maps to x1 and x2, y must map to
   different variables, not x1 and x2).
   - ONE-TO-MANY: A single Problem 1 variable can be mapped to multiple Problem
   2 variables (e.g., digit-wise $x = x_0 + 10x_1$ or split variables $x = x_1 
   + x_2$). Parts of a variable that cannot be detached must remain mapped to 
   that specific variable.
   - SEMANTIC PRIORITY: Variable descriptions are crucial. Analyze the physical
   meaning alongside the mathematical formulation.
3. **COEFFICIENT & MATHEMATICAL PRECISION:**
   - DECIMAL CONVERSION: All fractions (e.g., `1/10`, `frac{{1}}{{100}}`, 
   `1/100*j`) MUST be converted to decimal coefficients (e.g., `0.1`, `0.01`).
   - LINGUISTIC INVERSES: If the description states a Problem 2 variable is 
   "100 times before" or "100 times larger" than Problem 1, use the inverse 
   coefficient as **constant** (e.g., `0.01`) so that $P1\_var = 0.01 \times
   P2\_var$.
   - ALGEBRAIC CONSISTENCY: If you substitute your new mapping into Problem 1's 
   equations, they must become identical to Problem 2's equations.
4. **FORMAT:**
   - Output ONLY a valid raw JSON object. No markdown formatting, no triple 
   backticks, no explanations, no text outside the JSON.
   - Each mapping must be a list of objects: `{{"constant": float, 
   "variable": "string"}}`.
Example Format:
{{  "P1_VarA": [{{"constant": 1.0, "variable": "P2_VarX"}}],
  "P1_VarB": [{{"constant": 0.1, "variable": "P2_VarY"}}, {{"constant": 0.01,
  "variable": "P2_VarZ"}}]}}

\end{Verbatim}
\end{tcolorbox}
\caption{Prompt for Variable Mapping Correction}
\label{fig:prompt-variable}
\end{figure*}

\begin{figure*}
\centering
\tcbset{}
\begin{tcolorbox}[width=\textwidth]
\begin{Verbatim}
You are an expert AI assistant tasked with mapping parameters (constants) 
between two optimization problems.Your goal is to find the exact 
equivalent parameter in Problem 2 for EVERY parameter in Problem 1.To
help you, the CORRECT mapping of variables between these problems is provided.
You must also correct the PREVIOUSLY FAILED parameter mapping.
**Problem 1 (Reference):**
- Parameters: {json.dumps(params1, indent=2)} - Constraints: 
{json.dumps([c.get('formulation') for c in constraints1], indent=2)}
- Objective: {json.dumps(objective1.get('formulation', ''), indent=2)}
**Problem 2 (Target):** - Parameters: {json.dumps(params2, indent=2)}
-Constraints:{json.dumps([c.get('formulation')for c in constraints2],indent=2)}
- Objective: {json.dumps(objective2.get('formulation', ''), indent=2)}
**Correct Variable Mappings (Use this context to align the equations):**
{json.dumps(var_mappings, indent=2)}
**Previous Incorrect Parameter Mapping (TO BE CORRECTED):**
{json.dumps(old_param_mappings, indent=2)}
--- **STRICT RULES FOR PARAMETER MAPPING:** 1. **CORE MAPPING LOGIC:**
   - Every Problem 1 parameter MUST be mapped. Do NOT map to "none".
   - 1-TO-1 MAPPING: Unlike variables, parameters map exactly 1-to-1. Do not map
multiple different Problem 1 parameters to the exact same Problem 2 parameter.
   - SEMANTIC MATCHING: Look at the parameter descriptions first. Costs map to 
   costs, capacities map to capacities, limits map to limits. If words are not
   same look into the meanings they convey.
2. **USE THE VARIABLE MAPPING (ALGEBRAIC CONTEXT):**
   - You MUST use the provided `Correct Variable Mappings` to substitute
   Problem 2's variables into Problem 1's equations.
   - Use this substitution to see which parameters perfectly align in the newly
   balanced equations. 3. **COEFFICIENT RULE: 1.0 DEFAULT **
   - **DEFAULT TO 1.0:** In all cases, the constant for a parameter mapping is 
   exactly `1.0`. Do not invent scaling differences or guess coefficients. 
   (default 1.0 constant). If previously generated was not 1.0 change it to 1.0.
   - In rare cases it will be other value but if other value is submitted as old 
   parameter mapping, then there is high chance that constant is actually 1.0
   (1.0 strictly and only no other value)
4. **MANDATORY CORRECTION (NO REPETITION ALLOWED):**
   - The "Previous Incorrect Parameter Mapping" is entirely WRONG. It failed our 
   equivalence tests.
   - YOUR OUTPUT MUST CHANGE: You are strictly forbidden from outputting the
   exact same mapping you were given as the incorrect baseline. 
   - DIAGNOSE THE FAILURE: The old mapping likely failed because it paired the 
   wrong parameters, or it used a bizarre constant when it should have just been 
   `1.0` (or vice versa, failing to apply a rare mathematical exception). 
   Fix the error and output the corrected version. 5. **FORMAT:**
   - Output ONLY a valid raw JSON object. No markdown, no conversational text.
   - Each mapped value must be a list of objects representing the linear 
   relationship: `{{"constant": float, "parameter": "string"}}`.
Example Output Format (Showing standard 1.0 and a rare exception):
{{"P1_ParamA": [{{"constant": 1.0, "parameter": "P2_ParamX"}}],
  "P1_ParamB": [{{"constant": 0.1, "parameter": "P2_ParamY"}}]}}
\end{Verbatim}
\end{tcolorbox}
\caption{Prompt for parameter mapping correction}
\label{fig:prompt-parameter}
\end{figure*}

\begin{figure}[ht]
\centering
\tcbset{}
\begin{tcolorbox}
\begin{Verbatim}
You are an expert mathematical optimization solver and formal verification
engine.
Your task is to rigorously determine if two Mixed-Integer Programming (MIP) 
formulations (Problem A and Problem B) are strictly equivalent.
Two formulations are equivalent if and only if they share identical feasible 
regions and objective contours. This means they can be perfectly mapped to one 
another through:
1. 1-to-1 renaming of variables and parameters.
2. Algebraic rearrangement or scalar multiplication of constraints (e.g., 
$x + y \le 5$ is equivalent to $10 \ge 2y + 2x$).
3. Reordering of constraints or objective terms.
**Strict Parsing Rule:**
If a parameter or variable shape has only one value 'Integer', you must 
strictly treat it as a scalar entity, not a vector.
**First problem formulation (Problem A):**
{prob1_str}
**Second problem formulation (Problem B):**
{prob2_str}
**Instructions:**
Think step-by-step. Provide a brief analysis mapping the variables, 
parameters, objective functions, and constraints. 
After your analysis, you MUST output your final decision on the VERY LAST 
LINE of your response as exactly one of these two phrases:
Equivalent ; Not Equivalent
\end{Verbatim}
\end{tcolorbox}
\caption{COT-based Prompt for Equivalence Check}
\label{fig:prompt-gemini}
\end{figure}

\begin{figure}[ht]
\centering
\tcbset{}
\begin{tcolorbox}
\begin{Verbatim}
--- Instance from problem_info.json in problem 169 (169_c) ---
"variables": {
        "v": {
            "description": "The total amount of journeys made by cargo planes",
            "type": "continuous",
            "shape": [
                "integer"
            ]
        },
        "n": {
            "description": "The quantity of oversized truck journeys",
            "type": "continuous",
            "shape": [
                "integer"
            ]
        }
--- Instance from problem_info.json in problem 220 (220_c) ---
 "variables": {
        "c": {
            "description": "A binary variable that signifies if any Calcium 
            supplements are consumed.",
            "type": "continuous", "shape": ["Binary"]},
        "q": {
            "description": "Variable that is set to either 1 or 0 to show if any
            Vitamin D supplements have been consumed.",
            "type": "continuous", "shape": ["Binary"]},
        "f": {
            "description": "The overall duration required for the medications to 
            take effect.",
            "type": "continuous",  "shape": ["Continuous"]},
        "e": {
            "description": "The quantity of Calcium tablets consumed within a 
            month.",
            "type": "continuous","shape": ["Integer"]},
        "h": {
            "description": "The quantity of Vitamin D capsules consumed within 
            one month",
            "type": "continuous", "shape": [ "Integer"]}
            --------------------------------------------------
\end{Verbatim}
\end{tcolorbox}
\caption{Examples with irrelevant dimensions}
\label{fig:irrelevant dimension}
\end{figure}

\begin{figure}[ht]
\centering
\tcbset{}
\begin{tcolorbox}
\begin{Verbatim}
You are an expert mathematical optimizer and symbolic algebra engine. 
Your job is to find the exact non-linear algebraic mapping between an original
optimization problem(Formulation A) and its reformulated version(Formulation B).
### Context
We have two versions of the same optimization problem:
- **Formulation A**: Uses bounded or restricted variables.
- **Formulation B**: Uses unconstrained latent variables to simplify 
optimization.
### Data Inputs for Instance ID: {pair_id}
---
[FORMULATION A JSON]
{json_A_str}
---
[FORMULATION B JSON]
{json_B_str}
---
### Instructions
1. **Analyze Domains**: Look at the variable descriptions and constraints in
Formulation A (e.g., variable > 0, or 0 < variable < 1). Match them with
the unconstrained variables in 
Formulation B.
2. **Match Expressions**: Compare the Objective Functions and the Constraints. 
Identify how terms in Formulation A (like `log(CatalystDose)` or 
`MixingRatio / (1 - MixingRatio)`) map  directly to terms in Formulation
B (like `LatentDose` or `exp(LatentRatio)`).
3. **Derive the Forward Mapping**: Express each variable from Formulation A 
explicitly as a function of the variables in Formulation B.
4. **Derive the Inverse Mapping**: Invert your forward equations to express each 
variable from Formulation B explicitly as a function of the variables in
Formulation A.
### Output Format
Provide your final answer strictly in valid JSON format matching the schema 
below. Do not include any conversational prose, explanations outside the JSON,
or markdown code blocks (like ```json).
{{  "pair_id": "{pair_id}",
  "formulation_A_to_B": {{
    "VariableA_1": "expression_in_terms_of_B",
    "VariableA_2": "expression_in_terms_of_B"
  }},
  "inverse_mapping": {{
    "VariableB_1": "expression_in_terms_of_A",
    "VariableB_2": "expression_in_terms_of_A"
  }}
}}"""
\end{Verbatim}
\end{tcolorbox}
\caption{Prompt for non-linear mapping extraction}
\label{fig:dreal}
\end{figure}
\end{document}